\documentclass[letterpaper, 10 pt, conference, pdftex]{ieeeconf}  
\IEEEoverridecommandlockouts                              
\usepackage[pdftex]{graphicx}
\usepackage{epsfig} 
\usepackage{mathptmx} 
\usepackage{times} 
\usepackage{amsmath} 
\usepackage{amssymb}  
\usepackage{xcolor}

\newtheorem{theorem}{Theorem}
\newtheorem{lem}{Lemma}
\newtheorem{remark}{Remark}
\newtheorem{assumption}{Assumption}

\title{\LARGE \bf
Task-Space Singularity Avoidance for Control Affine Systems Using Control Barrier Functions
}

\author{Kimia Forghani\textsuperscript{1}, Suraj Raval\textsuperscript{1}, Lamar Mair\textsuperscript{2}, Axel Krieger\textsuperscript{3}, and Yancy Diaz-Mercado\textsuperscript{1}
\thanks{\textsuperscript{1} Authors are with the Department of Mechanical Engineering,
        University of Maryland, College Park, MD 20742
        {\tt\small \{kimiaf, sraval, yancy\}@umd.edu}}%
\thanks{\textsuperscript{2} Author is with the Division of Magnetic Manipulation and Particle Research, Weinberg
Medical Physics, Inc., North Bethesda, MD 20852
        {\tt\small lamar.mair@gmail.com}}%
\thanks{\textsuperscript{3} Author is with the Department of Mechanical Engineering, Johns Hopkins University,
Baltimore, MD 21218
        {\tt\small axel@jhu.edu}}%
}

\begin{document}

\maketitle
\thispagestyle{empty}
\pagestyle{empty}

\begin{abstract}
Singularities in robotic and dynamical systems arise when the mapping from control inputs to task-space motion loses rank, leading to an inability to determine inputs. This limits the system’s ability to generate forces and torques in desired directions and prevents accurate trajectory tracking. This paper presents a control barrier function (CBF) framework for avoiding such singularities in control-affine systems. 
Singular configurations are identified through the eigenvalues of a state-dependent input–output mapping matrix, and barrier functions are constructed to maintain a safety margin from rank-deficient regions. Conditions for theoretical guarantees on safety are provided as a function of the actuator dynamics. Simulations on a planar 2-link manipulator and a magnetically actuated needle demonstrate smooth trajectory tracking while avoiding singular configurations and reducing control input spikes by up to $100\times$ compared to the nominal controller.

\end{abstract}

\section{INTRODUCTION}
In many dynamical systems, particularly safety-critical applications, it is essential to have the ability to generate necessary forces and torques to achieve desired trajectories while maintaining stability and safety guarantees at all times. However, systems with state-dependent task spaces may encounter configurations in which this capability is compromised. In robotics, such configurations are known as kinematic singularities, where the mapping from control inputs to task-space outputs and the kinematics Jacobian loses rank (e.g., Fig.~\ref{fig:TwoLinkManipulators}). At a singular configuration, the actuation or control matrix becomes ill-conditioned, restricting the system’s ability to simultaneously generate forces and torques in desired directions. This state-dependent rank deficiency complicates motion control, as the system can no longer track arbitrary trajectories within its operating space. Moreover, near singularities, controllers that rely on matrix inversion often produce large and abrupt control inputs due to poor numerical conditioning. These spikes in actuation not only reduce control effectiveness but may also induce instability, limit performance, and pose safety concerns. 
\begin{figure}[t]
    \includegraphics[width=0.49\linewidth]{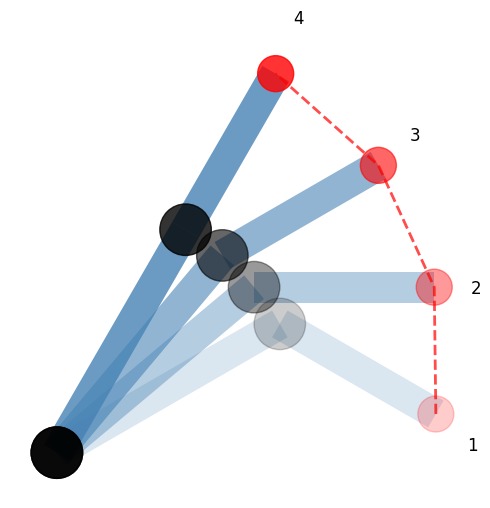}\hfill
    \includegraphics[width=0.49\linewidth]{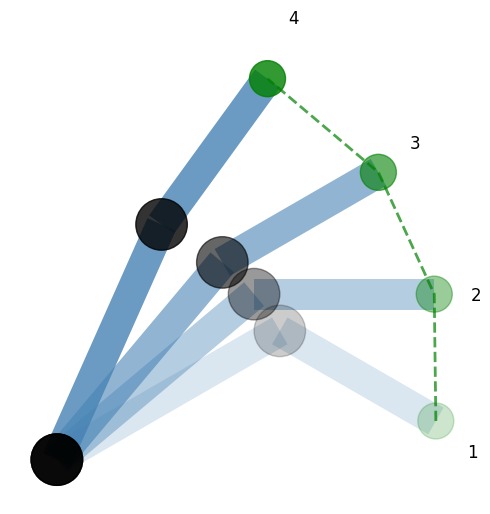}
    \caption{Two similar end-effector trajectories (dashed lines) for a two-link manipulator. The red trajectory in the fourth pose passes through a singular configuration when the middle joint angle becomes zero (flat links), while the green trajectory avoids singularity by slightly modifying the joint angles.}\label{fig:TwoLinkManipulators}
\end{figure}

Singularities in robotic and dynamical systems are well-recognized challenges in control, particularly due to their adverse effects on manipulability and real-time control performance \cite{siciliano2008springer}. Prior approaches include hybrid switching control \cite{tomlin1998switching}, hardware redesign \cite{wright2017spherical, chen2024mitigating}, Lie-bracket motion synthesis \cite{raval2024singularity, muller2018higher}, and regularization-based optimization methods \cite{johansen2004constrained, nguyen2020regularization}. While effective in specific settings, these methods either lack formal safety guarantees or are difficult to generalize. Recent barrier-function approaches \cite{kurtz2021control, wu2024singularity} begin to address this limitation for
particular system classes.

Control barrier functions (CBFs) provide a mathematical approach in control theory for maintaining constraints and securing system stability and safety \cite{ames2019control}. Expressed as a quadratic program (QP), they allow real-time computation of safe control actions. Their safety guarantees make them useful in various safety-critical applications, including collision avoidance \cite{11107732}, multi-agent systems \cite{zhang2025adaptive}, and medical \cite{forghani2025suture}. Here, we use CBFs to prevent the system from entering singular configurations accurately and efficiently. We establish a theoretical framework for identifying singularity points in a control affine system whose task-space is influenced by the control signal, applicable to underactuated, overactuated, and fully actuated setups. We present an analytical CBF construction and a numerical representation. Leveraging the framework, we adjust control inputs in real-time to avoid singular regions while minimizing deviation from the reference trajectory. Simulations validate the method.


The rest of the paper is organized as: Section~\ref{sec:Background} provides background, Section~\ref{sec:CBF} presents singularity identification, and CBF construction, Section~\ref{sec:Cases} presents case study and simulation results, and Section~\ref{sec:Conclusions} concludes.

\section{BACKGROUND AND PRELIMINARIES}\label{sec:Background}
\subsection{Dynamics and Task Output}
Consider an affine control system:
\begin{equation} \label{affine}
\begin{aligned}
    \dot{x} &= f(x) + g(x)u, \\
    z &= \gamma(x) 
\end{aligned}
\end{equation}
with \(x \in \mathbb{R}^n\) denoting the states, \(u \in \mathbb{R}^m\) being the control input, where \(f:\mathbb{R}^n\to\mathbb{R}^n\) and \(g:\mathbb{R}^m\to\mathbb{R}^{n\times m}\) are locally Lipschitz. The mapping between the task output \(z \in \mathbb{R}^d\) and the space is given by a continuously differentiable map \(\gamma:\mathbb{R}^n\to\mathbb{R}^{d}\).
The system is subject to input constraints \(u(t) \in\mathcal{U}\subseteq\mathbb{R}^m\), the space of admissible control signals. 

The task-space output changes according to
\begin{equation}
    \dot{z}=L_f\gamma(x)+L_g\gamma(x)u
\end{equation}
We define a state-dependent input-output mapping matrix according to the Lie derivative:
\begin{equation}
    \phi(x)=L_g\gamma(x)=\frac{\partial \gamma}{\partial x}g(x)=J(x)g(x) \in \mathbb{R}^{d\times m},
    \label{phi}
\end{equation}
where \(J(x) \in \mathbb{R}^{d\times n}\) is the Jacobian of the task-space output function.
A singular configuration occurs when the state-dependant input-output mapping matrix \(\phi(x)\) loses rank, i.e., when the number of singular values that approach zero is greater than \(\max(d,m)-\min(d,m)\). When this happens, the mapping between control inputs and states locally collapses the input space along certain directions, resulting in a loss of local invertibility. Such singularities are associated with critical points where the system may exhibit changes in behavior
because the commanded state cannot be achieved.

\subsection{Control Barrier Functions}

For the control-affine system in \eqref{affine}, define the safe set
\begin{equation}
    C = \{x \in \mathbb{R}^n : h(x) \geq 0\},
\end{equation}
where \(h(x)\) is continuously differentiable. Safety is ensured if \(C\) is forward invariant \cite{ames2019control}.

A function \(h(x)\) is a Control Barrier Function (CBF) if there exists an extended class-\(\mathcal{K}\) function \(\alpha\) such that
\begin{equation}
    L_f h(x) + L_g h(x) u \ge -\alpha(h(x)).
\end{equation}

Since $\dot{h}(x,u) = L_f h(x) + L_g h(x) u$ is affine in the control input, safety constraints can be incorporated into a quadratic program (QP). Given a nominal control input $v(x)$, the CBF-QP minimally modifies it to enforce safety:
\begin{equation}
\begin{aligned}
    u(x) &= \arg\min_{u \in \mathbb{R}^m} \tfrac{1}{2} \|u - v(x)\|^2 \\
    &\text{s.t. } \; L_f h(x) + L_g h(x) u \ge -\alpha(h(x)).
    \label{cbf1}
\end{aligned}
\end{equation}

\section{CONTROL BARRIER FUNCTION CONSTRUCTION}\label{sec:CBF}
\subsection{Quadratic Problem}

Equation \eqref{cbf1} enforces safety while minimizing deviation from the desired control input. Here, we ensure minimal deviation from the desired task-space kinematics \(\dot z_d\). As \eqref{affine}:
\begin{equation}
 \dot z_{d}=L_f\gamma(x) +\phi(x)v(x), 
\end{equation}
where \(v(x) \in \mathbb{R}^m\) is the reference or nominal control input. Since the deviation of the error between the desired and true task-space kinematics is given by
\begin{align*}
    \|\dot{z}-\dot{z}_d\|^2 = (u-v(x))^T\phi^T(x)\phi(x)(u-v(x))
\end{align*}
We modify the QP formulation as:
\begin{equation}
\begin{aligned}
    u &= \arg\min_{u \in \mathbb{R}^m}  (u-v(x))^TQ(x)(u-v(x))\\
    &\text{s.t. } \; L_f h(x) + L_g h(x) u \geq -\alpha(h(x)).
    \label{cbf}
\end{aligned}    
\end{equation}
To ensure \(\phi(x)\) is kept at full rank, we define
\begin{align}
    Q(x) = 
    \begin{cases}
        \phi^T(x)W\phi(x), & d\geq m\\
        \phi^T(x)W\phi(x) + \Gamma, & d<m
    \end{cases}
\end{align}
for some positive definite matrices \(W\in\mathbb{R}^{d\times d}\), \(\Gamma\in\mathbb{R}^{m\times m}\). So long as $\phi(x)$ is full rank, this ensures that \(Q(x)\) is positive definite, and the QP cost is convex.


\begin{remark}
For over-actuated systems (\(d<m\)), \(\phi^T\phi\) is rank deficient, so a regularization term \(\Gamma\succ0\) is added to ensure convexity \cite{tikhonov1977solutions}. See section \ref{sec:Cases}.B.
\end{remark}


\begin{assumption}
The reference policy \(v(x)\) stabilizes the dynamics in \eqref{affine}, and the desired output \(z_d\) is reachable under admissible inputs, particularly in underactuated cases.
\end{assumption}

\subsection{Singularity Points}
Singular configurations occur when the input–output mapping matrix \(\phi(x)\) in \eqref{phi} becomes rank deficient. This framework applies to underactuated, fully actuated, and over-actuated systems, independent of the chosen reference controller. Since \(\phi(x)\) may be non-square, we define matrix \(M(x)=\phi(x)\phi(x)^T\) for over-actuated systems and \(M(x)=\phi(x)^T\phi(x)\) otherwise. The eigenvalues of \(M(x)\) are real, non-negative, and equal to the squared singular values of \(\phi(x)\). This  allows a unified representation since a rectangular $\phi \in \mathbb{R}^{n \times m}$ has only $\min(n,m)$ singular values, and $M(x)$ preserves these values through its eigen-structure.

Without loss of generality, we take the case where the system is over-actuated.
In this case, the \(M(x)=\phi(x)\phi(x)^T\) matrix is always a  \(\mathbb{R}^{d\times d}\) square matrix, regardless of the number of actuators. To find its eigenvalues, we solve the characteristic equation:
\begin{equation}
det(\lambda I-\phi(x)\phi(x)^T)=0
\label{chareq}
\end{equation}

Expanding the determinant in \eqref{chareq} leads to:
\begin{equation}
p(\lambda) = \lambda^d + c_{d-1}\lambda^{d-1} + \cdots + c_1 \lambda + c_0,
\end{equation}
where the coefficients $c_i$ depend on the entries of $\phi(x)\phi(x)^\top$.

Analytically solving for $p(\lambda)=0$ gives an explicit expression for each eigenvalue as a function of the states. This characterizes the level set in $\mathbb{R}^d$ where $\lambda$ becomes zero. Subsequently, computing the derivative with respect to all states yields $d$ conditions for constructing the CBF. 

\subsection{Barrier Function}

We define the safe set associated with each CBF as:
\begin{equation}
    \mathcal{C}_i = \left\{ x \in \mathbb{R}^d \mid \lambda_i(x) - \epsilon \geq 0 \right\},
\end{equation}
which corresponds to the subset of the state space where the \( i \)-th eigenvalue of \( \phi(x)\phi(x)^\top \) remains greater than zero and \( \epsilon > 0 \) is a small positive margin to enforce a buffer away from singularities. 
A candidate CBF for each eigenvalue is:
\begin{equation}
    h_i(x) = \lambda_i(x) - \epsilon,
\end{equation}

\begin{lem}[\cite{magnus1985matrix}]
Let \( \phi \in \mathbb{R}^{d \times m} \) be a differentiable matrix-valued function. Then the matrix \( M(x) = \phi(x)\phi(x)^\top \) is symmetric, positive semidefinite, and differentiable.
\end{lem}

\begin{lem}[\cite{lancaster1985theory}]
All simple eigenvalues \( \lambda_i(x) \) of a symmetric, differentiable matrix \(M(x)\) are differentiable functions of \( x \).
\end{lem}

\begin{theorem}
\label{th1}
For the system in \eqref{affine}, let \( M(x) = \phi(x)\phi(x)^\top \in \mathbb{R}^{d \times d} \) and suppose \( M(x) \) has eigenvalues \( \lambda_1(x), \dots, \lambda_d(x) \). Then, for any \( \epsilon > 0 \) and \( x(0) \in \mathcal{C} \), the functions
\[
h_i(x) = \lambda_i(x) - \epsilon, \quad i = 1, \dots, d,
\]
are valid \textit{barrier functions} and the control in \eqref{cbf} renders the sets
\[
\mathcal{C}_i = \{ x \in \mathbb{R}^d \mid h_i(x) \geq 0 \} = \{ x \in \mathbb{R}^d \mid \lambda_i(x) \geq \epsilon \}.
\]
forward invariant so long as either \(\nabla \lambda_i(x)^\top g(x) \neq 0\) or \(\nabla \lambda_i(x)^\top f(x) \geq -\alpha(h_i(x))\) for every \(x\in\mathcal{C}\).
\end{theorem}

\begin{proof}
Since both \( J(x) \) and the input matrix \( g(x) \) are assumed continuously differentiable, the mapping 
\(\phi(x) = J(x) g(x)\) is differentiable. Consequently by Lemma 1 the matrix \( M(x) = \phi(x)\phi(x)^\top \) is symmetric, positive semidefinite, and differentiable. By Lemma 2, all eigenvalues \( \lambda_i(x) \) of \( M(x) \) are differentiable and real. Thus, each function \( h_i(x) = \lambda_i(x) - \epsilon \) is differentiable. Using the system dynamics, we compute the derivative along trajectories:
\[
\dot{h}_i(x) = \nabla \lambda_i(x)^\top \dot{x} = \nabla \lambda_i(x)^\top \left( f(x) + g(x) u \right),
\]
yielding the condition
\begin{equation}\label{cond}
\nabla \lambda_i(x)^\top f(x) + \nabla \lambda_i(x)^\top g(x) u + \alpha(\lambda_i(x) - \epsilon) \geq 0.
\end{equation}

This inequality is affine in \( u \)  and is feasible if there exists an admissible \( u \in \mathbb{R}^m \) satisfying it. In particular, feasibility is guaranteed if either
\begin{equation}\label{cond2}
    \nabla \lambda_i(x)^\top g(x) \neq 0,
\end{equation}
so that the control input can influence the inequality, or, if \( \nabla \lambda_i(x)^\top g(x) = 0 \), the inequality is still feasible provided
\[
\nabla \lambda_i(x)^\top f(x) \geq -\alpha(h_i(x)),
\]
i.e., the drift term naturally enforces safety.





If either case is valid for every \(x\in\mathcal{C}\), this guarantees that all eigenvalues remain bounded away from zero, maintaining full rank of \( \phi(x) \) and thus avoiding singular configurations.
\end{proof}

\begin{remark}
Even when condition \eqref{cond2} is satisfied the feasibility of enforcing \eqref{cond}
is not guaranteed in the presence of control constraints such as bounded inputs or actuation saturation. In such cases, the validity of the CBF should be evaluated on a case-by-case basis by verifying whether the required control input lies within the admissible set \( \mathcal{U} \subset \mathbb{R}^m \). 
\end{remark}

\begin{remark}
Constructing a valid CBF for complex systems can be tedious due to extensive algebraic manipulation. An alternative is to approximate the derivative numerically at each time step. This simplifies system-dependent derivations but is sensitive to noise and disturbances, leading to numerical instabilities. Instead, we reformulate the problem as an obstacle avoidance problem, as described in section IV.B through an example. This approach eliminates the need to solve for eigenvalue conditions and uses numerical Singular Value Decomposition (SVD). 
\end{remark}

\section{CASE STUDIES}\label{sec:Cases}
This section explains the general framework for constructing CBFs for any dynamical system to avoid singularities using analytical and numerical approaches.

\subsection{Analytical CBF}
Here we provide an example of a simple control system for which we can do an analytical derivation. 
\paragraph{System description} Consider a 2--link planar manipulator (Fig.\ref{fig:TwoLinkManipulators}). Let the joint angles and control inputs be 
\[
q = \begin{bmatrix} q_1 \\ q_2 \end{bmatrix} \in \mathbb{R}^2, 
\qquad u = \dot{q} \in \mathbb{R}^2.
\] The joint--space dynamics are written in control--affine form:
\begin{align}
\dot{q} &= f(q) + g(q)u, \\
f(q) &= 0, \qquad g(q) = I_2.
\end{align}

The end--effector position, task output, is obtained via forward kinematics:
\begin{align}
z(q) =
\begin{bmatrix}
x(q) \\ y(q)
\end{bmatrix}
=
\begin{bmatrix}
l_1 \cos q_1 + l_2 \cos(q_1 + q_2) \\
l_1 \sin q_1 + l_2 \sin(q_1 + q_2)
\end{bmatrix},
\end{align}
where $l_1,l_2$ are the link lengths. This yields the task--space velocity relation
\begin{align}
 & \quad \quad \quad \quad \dot{z}= J(q)I_{2\times 2}\,u=J(q)\,u,\\
J(q) &=
\begin{bmatrix}
- l_1 \sin q_1 - l_2 \sin(q_1+q_2) & - l_2 \sin(q_1+q_2) \\
\;\; l_1 \cos q_1 + l_2 \cos(q_1+q_2) & \;\; l_2 \cos(q_1+q_2)
\end{bmatrix}.
\end{align}

\paragraph{Control Barrier Function}
To detect singularities we analyze the eigenvalues of $JJ^\top$. Solving \eqref{chareq} with \(l_1=l_2=1\) yields two eigenvalues $\lambda_1(q_2), \lambda_2(q_2)$, with $\lambda_2 > \lambda_1$:

\begin{align}
\lambda_1 &= \cos q_2 - \tfrac{1}{2}\sqrt{12\cos q_2 + 8\cos^2 q_2 + 5} + \tfrac{3}{2}, \\
\lambda_2 &= \cos q_2 + \tfrac{1}{2}\sqrt{12\cos q_2 + 8\cos^2 q_2 + 5} + \tfrac{3}{2}.
\end{align}
We define the CBF as:
\begin{equation}
h(q) = \lambda_1(q_2) - \epsilon, \qquad \epsilon > 0.
\label{eq:cbf}
\end{equation}
To enforce the constraint $\dot{h}(q,u) \geq -\alpha( h(q))$, we compute
\begin{align}
\dot{h}(q,u) &= \nabla h(q)^\top I_2\dot{q} 
= \begin{bmatrix} 0 & \tfrac{\partial\lambda_1}{\partial q_2} \end{bmatrix} u, \\
\frac{\partial\lambda_1}{\partial q_2} 
&= -\sin q_2 + \frac{12\sin q_2 + 16\cos q_2 \sin q_2}
{4\sqrt{12\cos q_2 + 8\cos^2 q_2 + 5}}.
\label{lambdagrad}
\end{align}

\paragraph{Feasibility Guarantee}
Based on Theorem \ref{th1}, the QP is feasible if either \( L_g h \neq 0 \) or \( L_f h + \alpha(h) \ge 0 \). That is:
\begin{equation}\textit{ either \quad }
\nabla h(q)^\top I_2 
= \begin{bmatrix} 0 & \tfrac{\partial\lambda_1}{\partial q_2} \end{bmatrix} \neq 0 \textit{\quad or\quad } \alpha(h) \ge 0 
\end{equation}
Solving for \eqref{lambdagrad} equal to zero on \(0 \leq q \leq2\pi \) yields \(q=0,\frac{2\pi}{3},\pi,\frac{4\pi}{3}\). Since $L_f h=0$ here, feasibility holds whenever $\alpha(h)=h\ge 0$, including at points where $\frac{\partial\lambda_1}{\partial q_2}=0$
(e.g., $q_2\in\{2\pi/3,\,4\pi/3\}$). The configurations for \(q=0,\pi\) are on the safety boundary and therefore will not be reached.

\paragraph{Quadratic Program}
At each time step, a simple proportional controller is used
to generate the task--space controller (desired end-effector velocity), and the corresponding reference joint velocity is obtained by mapping it using the pseudoinverse of the Jacobian.
\begin{align}
\dot{z}_{\text{des}} &= K_p (z_d - z), \\
u_{\text{ref}} &= J^\dagger(q)\,\dot{z}_{\text{des}},
\end{align}
where $z_d$ is the target position. The actual input $u$ is given by the solution of the QP
\begin{equation}
\begin{aligned}
\min_{u \in \mathbb{R}^2} \quad & \| J(q)u - J(q)u_{\text{ref}} \|^2 \\
\text{s.t.} \quad & \dot{h}(q,u) \geq -\alpha (h(q)). \label{eq:qp_cbf}
\end{aligned}
\end{equation}

The cost minimizes deviation from the reference while the constraint keeps $\lambda_1(q)$ positive, thereby avoiding singular configurations.

\paragraph{Simulation}
We consider two sequential tasks. The first drives the manipulator toward a singular configuration $(x,y) = (\sqrt{2}, \sqrt{2})$, which corresponds to holding $q_1$ fixed while letting $q_2 \to 0$ (the arm straightens). The second task drives the manipulator to a non-singular target. As shown in Fig.~\ref{fig:theta2_cbf}, without the CBF the arm attempts to fully straighten, causing large spikes in the joint velocities due to Jacobian ill-conditioning. With the CBF constraint \eqref{eq:qp_cbf}, the controller modifies $u$ to keep $q_2$ away from singularity. This reduces control spikes by approximately $70\times$ and $100\times$ for joints 1 and 2, respectively, while producing smoother joint trajectories.

\begin{figure}[!t]
    \centering
    \includegraphics[width=\linewidth]{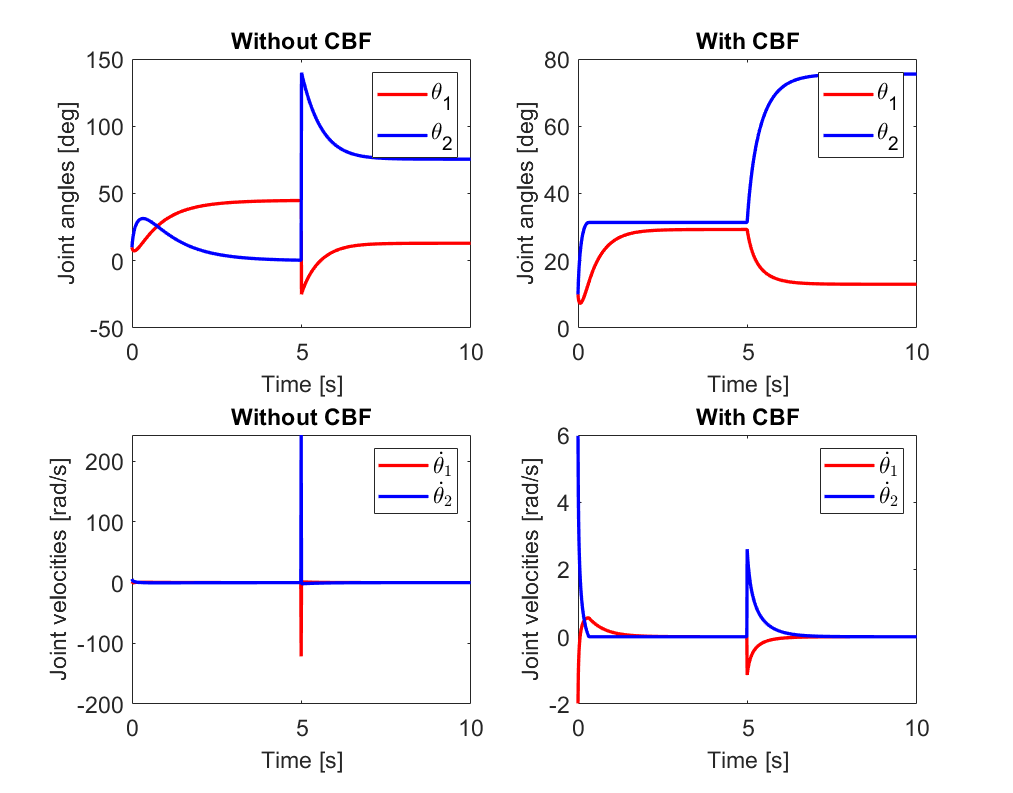}
    \caption{Comparison of joint angles and velocities: (Left) without the CBF,  $\theta_2 \to 0$  at \(t=5s\) induces spikes in joint velocities and abrupt changes in joint angles. (Right) with the CBF,  the singularity is avoided, the joint angle trajectory remains smooth, and the joint velocities remain low.}
    \label{fig:theta2_cbf}
\end{figure}

\subsection{Numerical CBF}

\begin{figure}[b]
      \centering
     
      \includegraphics[width=\linewidth]{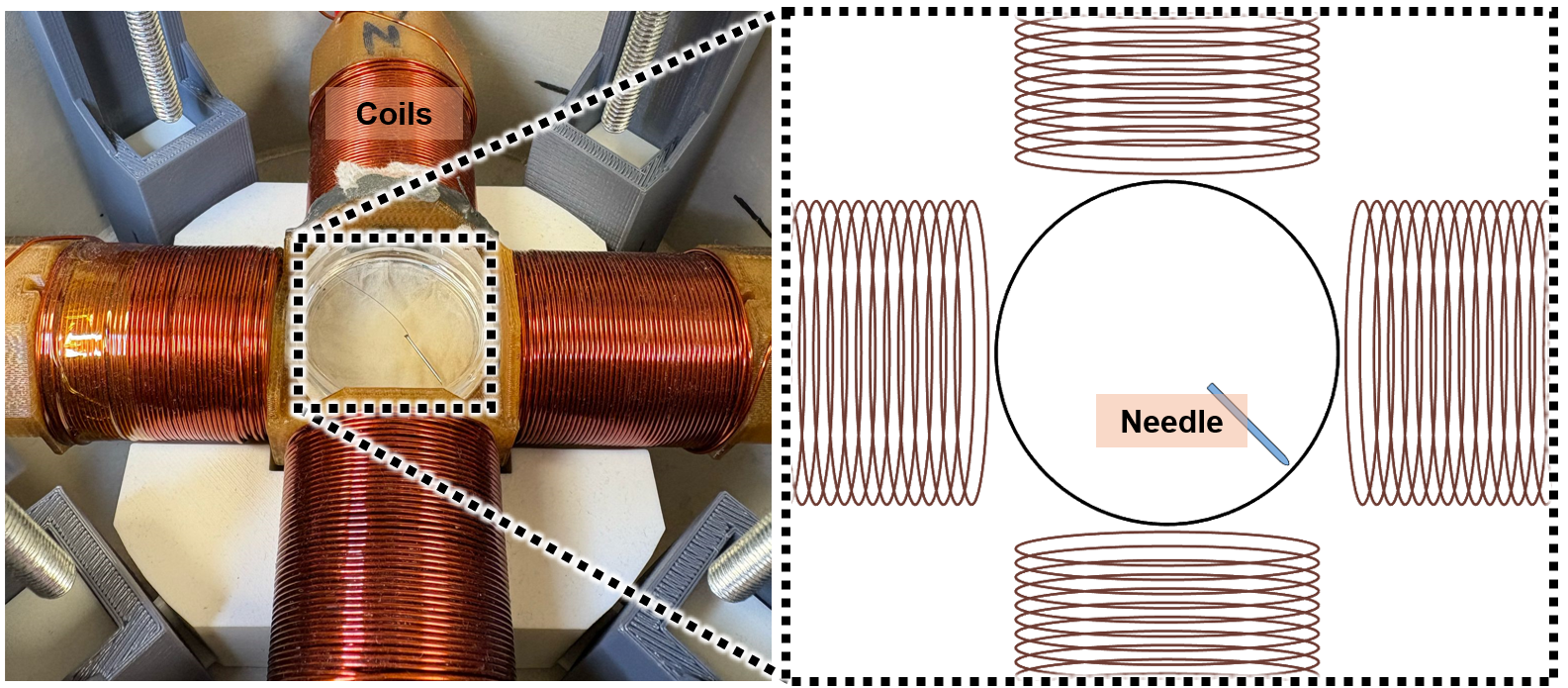}
    
      \caption{The magnetic actuation system (left), and the corresponding simulation workspace in MATLAB (right).}
      \label{fig:magneto}
   \end{figure}
The preceding example shows that CBFs can be constructed directly from analytical derivatives of the eigenvalues. However, for  complex systems, the characteristic equation often yields expressions that are algebraically intractable. We next present an implementation in which the singular set is characterized numerically, while maintaining the same CBF-QP objective.

\paragraph{System description}   
We consider a magnetic actuation system consisting of 4 stationary electromagnets arranged around a  workspace to control the position of a magnetic agent as depicted in Fig. \ref{fig:magneto}. Addressing singularities in magnetic actuation has been a key research focus due to their impact on controllability and stability \cite{yang2020magnetic}.
Under some mild assumptions \cite{fan2020towards}, the dynamics are modeled as a first order driftless system. The general dynamics is written as the control-affine system:
\begin{align}\label{eq:dyn}
\dot{z}=\dot{X}=g(X)u,
\end{align}
Here, \( \dot{X} \) represents both the state velocity vector of the magnet and its task output, and \( u \) is the control input.
The dynamics are expanded as:
\begin{equation}\label{g}
\begin{bmatrix}
c_t \dot{x} \\ c_t \dot{y} \\ c_r \dot{\theta}
\end{bmatrix}
=
\begin{bmatrix}
F_{x,1} & F_{x,2} & F_{x,3} & F_{x,4} \\
F_{y,1} & F_{y,2} & F_{y,3} & F_{y,4} \\
\tau_1 & \tau_2 & \tau_3 & \tau_4
\end{bmatrix}
\begin{bmatrix}
I_1 \\ I_2 \\ I_3 \\ I_4
\end{bmatrix},
\end{equation}
where the \(c\) are friction coefficients, \( F_{x,i} \) and \( F_{y,i} \) are the forces in the \( x \) and \( y \) directions, respectively, and \( \tau_i \) is the torque generated by the \( i \)-th electromagnet, \(i=1,\ldots,4\), when supplied with 1 A of current. The control inputs are the currents \( I_i \) applied to each electromagnet. The magnetic agent is modeled as a magnetic dipole with moment $m$. The resulting torque and force acting on the agent are given by:
\begin{equation}
\tau = m \times B,
\qquad F = (m \cdot \nabla) B,
\end{equation}
where \(B\) is the magnetic field at the agent’s location. The magnetic field $B$ and its gradient are computed using the dipole approximation of each coil \cite{griffiths2023introduction}.

We are interested in determining singular configurations by identifying when the output matrix \(g\) defined in \eqref{g} becomes ill-conditioned. Since \(g\) is a non-square matrix, we consider \( g g^T \succ 0 \), whose eigenvalues are real, non-negative, and equal to the squares of the singular values of \(g\):
\begin{equation}
    g g^T = \begin{bmatrix} 
        \sum_{i=1}^{n} F_{x,i}^2 & \sum_{i=1}^{n} F_{x,i} F_{y,i} & \sum_{i=1}^{n} F_{x,i} \tau_i \\
        \sum_{i=1}^{n} F_{x,i} F_{y,i} & \sum_{i=1}^{n} F_{y,i}^2 & \sum_{i=1}^{n} F_{y,i} \tau_i \\
        \sum_{i=1}^{n} F_{x,i} \tau_i & \sum_{i=1}^{n} F_{y,i} \tau_i & \sum_{i=1}^{n} \tau_i^2 
    \end{bmatrix}
\end{equation}

Although by solving the characteristic equation in \eqref{chareq} we have an explicit expression for the safe regions in terms of the system states \((x,y,\theta)\), it is extremely lengthy and impractical, especially since constructing a CBF also requires computing its derivatives. To address this, we represent the unsafe regions as virtual 3D obstacles and reformulate the problem as a standard obstacle avoidance task. To create these obstacles we first create point clouds of unsafe configurations using the MATLAB built-in SVD function to numerically find configurations corresponding to near-zero singular values. The corresponding point cloud is illustrated in Fig. \ref{fig:singularity-cloud} 
\begin{figure}[t]
      \centering
     
      \includegraphics[width=3.2in]{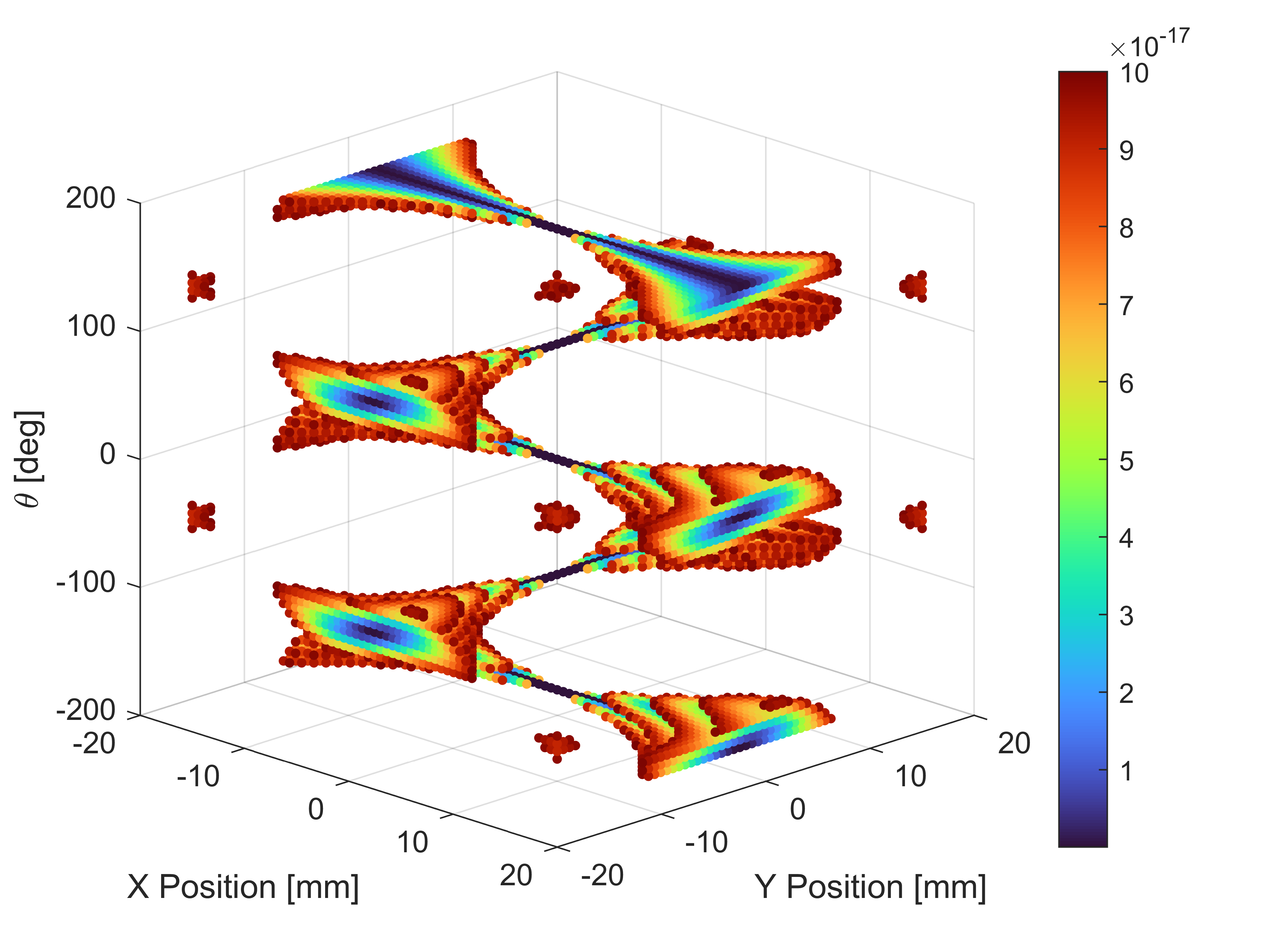}
    
      \caption{3D scatter plot of configurations where the smallest singular value falls below a threshold for the magnetic system in Fig.~\ref{fig:magneto}. Color indicates singular value magnitude. Coordinates represent 2D position and orientation $\theta$ of the magnetic agent.}

      \label{fig:singularity-cloud}
   \end{figure}
\begin{figure}[t]
    \centering
    \includegraphics[width=0.49\textwidth]{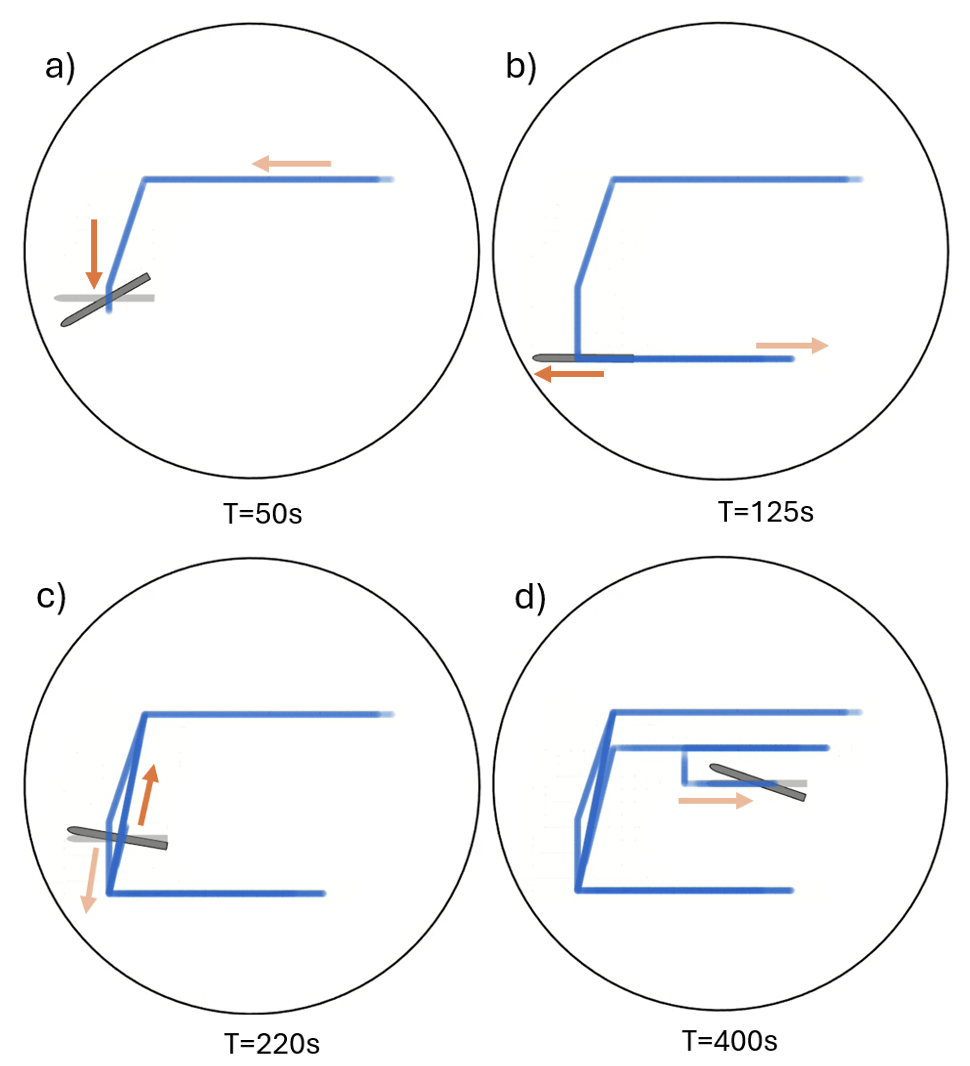}
      \caption{Snapshots of a simulated suturing task. The needle (gray) follows the planned trajectory (blue) while tilting from the reference orientation (gray shadow) to avoid singular configurations. Orange arrows indicate motion direction. The needle: a) approaches the first penetration site, b) tightens the suture while adjusting orientation, c) advances along the incision returning to the reference angle, d–f) completes the stitch.}\label{fig:casestudy}
\end{figure}

\paragraph{3D Obstacle Construction}
Distance-based constraints are often implemented by approximating obstacles as ellipsoids \cite{ferraguti2020control} or hyperspheres \cite{zeng2021safety} to ensure differentiability. However, such approximations significantly overestimate concave regions. Instead, we partition the sampled point cloud into multiple local components, using methods described below to guarantee local differentiability.

Given a 3D point cloud \(P\), we first scale it into a unit cube for numerical stability and reconstruct its surface boundary using the alpha shape method \cite{edelsbrunner1994three}. The resulting surface represents the zero level set of the barrier function and defines the safety boundary. The reconstruction partitions the point cloud into \(M\) distinct components. The alpha shape method preserves concave features through a tunable parameter \(\alpha\), producing a triangulated boundary that closely matches the sampled geometry. This triangulation ensures local continuity of the distance-based CBF near the safety boundary. Disjoint points are treated as separate obstacles to prevent undesirable switching behavior.

\paragraph{Safety-Set Construction}
The closest point to a query point \(p\) is computed by projecting onto each triangle and selecting the minimum-distance candidate; if the projection lies outside a triangle, the closest point is computed on its edges. The computation is implemented using vectorized operations to efficiently handle large surface meshes. Although the closest point may switch abruptly in concave regions, this does not affect control since the CBF is active only near the boundary. To allow greater control freedom away from singular regions, we choose a less restrictive class-\( \mathcal{K} \) function in our CBF construction. In particular, the choice \( \alpha(h) = \gamma h^2 \) yields larger numerical value than \( \alpha(h) = \gamma h \) for large values of \( h \geq 1\). Thus, it relaxes the CBF constraint and reduces unnecessary control intervention away from the boundary. 

For agent at position \(X\) and obstacles $O = 1, \dots, m$, we define an obstacle avoidance function as:
\begin{equation}
h_{\text{obs,}i, O}(X, p) = \tfrac{1}{2} (\|X - p_O\|^2 - \delta^2),\label{eq:safety_function}
\end{equation}
where \(p_O\) is the position of the closest point on obstacle \(O\), and \(\delta\) is the minimum allowable distance between the two.

\paragraph{CBF Controller}
Our objective is to design a control input that tracks a desired state velocity while guaranteeing safety via CBFs. To achieve this, we penalize the deviation of the induced state velocity \( \dot{X} = {g(x)} u \) \eqref{eq:dyn} from a reference velocity \( \dot{X}_{\text{ref}} \) through the cost
\begin{equation}
(u - u_{\text{ref}})^\top \left( {g(x)^\top} W {g(x)}\right) (u - u_{\text{ref}})
\end{equation}
where \( W \succ 0 \) is a weight matrix penalizing errors in state velocities. We weight \( x \) and \( y \) more heavily than \( \theta \).

Since \( {g(x)^\top} W {g(x)} \) has rank at most 3 due to over-actuation, the resulting cost matrix is rank-deficient. We resolve this by introducing a regularization term \( \Gamma \succ 0 \) \cite{tikhonov1977solutions} and define:
\begin{equation}
Q(x) = \left( {g(x)^\top} W {g(x)} + \Gamma \right)
\end{equation}

The control input is computed by the following QP:
\begin{equation}
\begin{aligned}
\min_{\substack{u \in \mathbb{R}^{d}}} \;\; \tfrac{1}{2} &(u - u_{\text{ref}})^\top Q(x) (u - u_{\text{ref}})  \\
\text{s.t.} \quad &A_{\text{obs}} u \geq b_{\text{obs}} 
\label{eq:qp_control}
\end{aligned}
\end{equation}
where 
\( u \in \mathbb{R}^d \) is the control input,
\( u_{\text{ref}} \in \mathbb{R}^d \) is the reference control input,
\( A_{\text{obs}}= \frac{\partial h_m}{\partial x} \in \mathbb{R}^{m \times d}  \), and
\( b_{\text{obs}} = -\alpha( h_m)=h_m(x)^2\in \mathbb{R}^{m} \) .


\paragraph{Simulation}
As a case study, we simulate a small-incision suturing task \cite{wang2025semi}. The objective is to steer a magnetically actuated needle along a stitch trajectory while avoiding singular configurations. The simulation represents the magnetic actuation environment in Fig. \ref{fig:magneto} with workspace diameter of 35 mm.

Fig.~\ref{fig:casestudy} shows representative task snapshots. As the needle approaches singular regions, the CBF-QP modifies its orientation while maintaining the desired spatial path. This adjustment prevents rank deficiency of the actuation matrix without missing penetration points. Because suturing is performed on flexible tissue, these small orientation deviations do not degrade stitch quality. Once the unsafe region is cleared, the needle smoothly returns to the reference orientation. The RMS position error is 0.99 mm, the RMS orientation error is 0.65 rad, and the maximum current applied is 4 A.

The corresponding trajectory in configuration space is shown in Fig.~\ref{fig:casestudydata}, where singular configurations are depicted in green. The colored path demonstrates that the needle remains outside the singular set while suturing.

      \begin{figure}[t!]
      \centering
     
      \includegraphics[width=\linewidth,clip=true,trim=150 0 100 120]{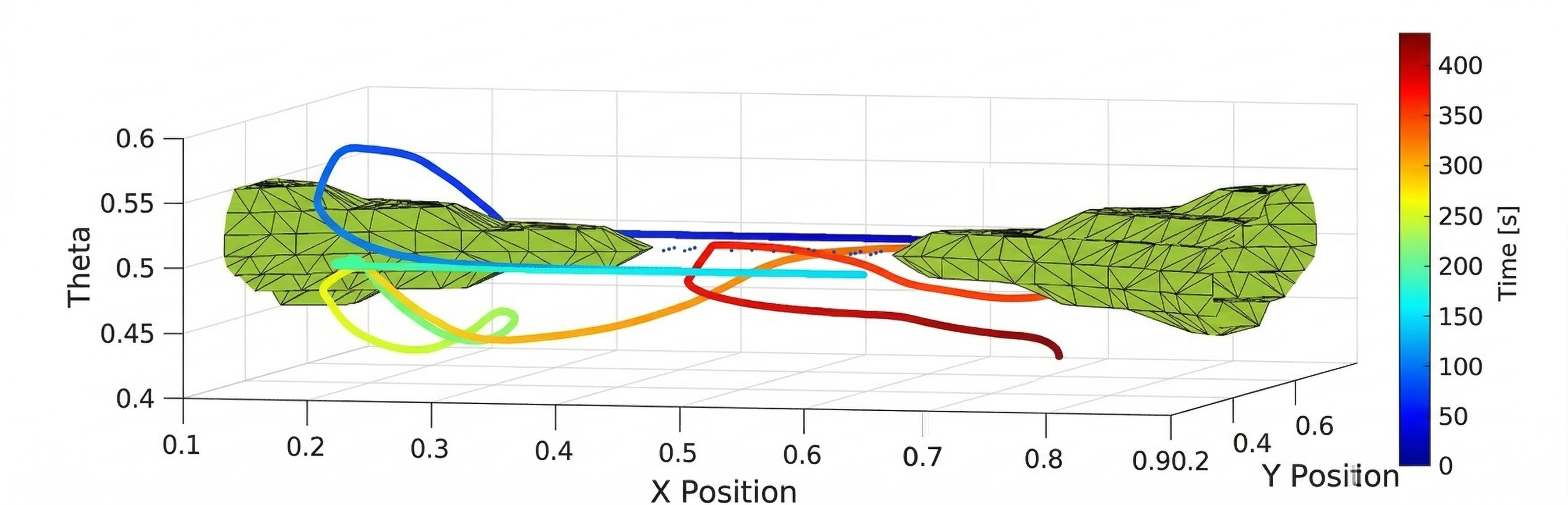}

      \caption{Safe magnetic needle trajectory (time-colored dots) for a suture pattern in configuration space that avoids singular configurations (green). }\label{fig:casestudydata}
   \end{figure}

\section{CONCLUSIONS}\label{sec:Conclusions}
In this work, we introduce a CBF-based framework to address configuration-dependent singularities in time-varying systems, a key challenge in robotics and control. Singularities are identified and represented analytically and numerically, and used to perform real-time avoidance via a CBF-QP with minimal trajectory deviation. Results demonstrate improved control smoothness and stability, reducing abrupt input fluctuations near singularities, ensuring safer operations.



\bibliographystyle{IEEEtran}
\bibliography{references}

\end{document}